\documentclass{article}

\usepackage{microtype}
\usepackage{graphicx}
\usepackage{subfigure}
\usepackage{booktabs} 

\usepackage{hyperref}

\usepackage[accepted]{icml2025}

\usepackage{amsmath}
\usepackage{amssymb}
\usepackage{mathtools}
\usepackage{amsthm}
\usepackage{bm}
\usepackage{caption}
\usepackage{subcaption}
\usepackage{wrapfig}

\usepackage[capitalize,noabbrev]{cleveref}

\theoremstyle{plain}
\newtheorem{theorem}{Theorem}[section]
\newtheorem{proposition}[theorem]{Proposition}

\theoremstyle{definition}

\theoremstyle{remark}

\usepackage[textsize=tiny]{todonotes}

\newcommand{\E}{\mathbb{E}} 

\newcommand{\W}{{\bm W}}

\newcommand{\D}{{\bm D}}

 \newcommand{\z}{{\bm z}}
 \newcommand{\x}{{\bm x}}

\newcommand{\res}{{\bm r}}
\newcommand{\smlb}{{\bm b}}

\DeclareMathOperator*{\argmax}{arg\,max}

\icmltitlerunning{Evaluating Sparse Autoencoders: From Shallow Design to Matching Pursuit}

\begin{document}

\onecolumn
\icmltitle{Evaluating Sparse Autoencoders: From Shallow Design to Matching Pursuit}



\icmlsetsymbol{equal}{*}
\icmlsetsymbol{senior}{$\dagger$}

\begin{icmlauthorlist}
\icmlauthor{Valérie Costa}{a1,equal}
\icmlauthor{Thomas Fel}{a2,a3}
\icmlauthor{Ekdeep Singh Lubana}{a2,a5}
\icmlauthor{Bahareh Tolooshams}{a4,a6,a7,senior}
\icmlauthor{Demba Ba}{a2,a3,senior}
\end{icmlauthorlist}

\icmlaffiliation{a1}{École Polytechnique Fédérale de Lausanne}
\icmlaffiliation{a2}{Harvard University}
\icmlaffiliation{a3}{Kempner Institute}
\icmlaffiliation{a4}{Department of Electrical and Computer Engineering, University of Alberta}
\icmlaffiliation{a5}{CBS-NTT Program in Physics of Intelligence, Harvard University}
\icmlaffiliation{a6}{Alberta Machine Intelligence Institute (Amii)}
\icmlaffiliation{a7}{Neuroscience and Mental Health Institute, University of Alberta}

\icmlcorrespondingauthor{Demba Ba}{demba@seas.harvard.edu}

\icmlkeywords{Machine Learning, ICML}

\vskip 0.15in


\printAffiliationsAndNotice{}  

\begin{abstract}

Sparse autoencoders (SAEs) have recently become central tools for interpretability, leveraging dictionary learning principles to extract sparse, interpretable features from neural representations whose underlying structure is typically unknown. This paper evaluates SAEs in a controlled setting using MNIST, which reveals that current shallow architectures implicitly rely on a quasi-orthogonality assumption that limits the ability to extract correlated features. To move beyond this, we compare them with an iterative SAE that unrolls Matching Pursuit (MP-SAE), enabling the residual-guided extraction of correlated features that arise in hierarchical settings such as handwritten digit generation while guaranteeing monotonic improvement of the reconstruction as more atoms are selected.

\end{abstract}

\section{Introduction}\label{introduction}

Sparse dictionary learning \cite{mairal_online_2009,rubinstein_dictionaries_2010,tosic_dictionary_2011} aims to represent data $\x^i$ as sparse linear combinations of learned basis vectors ($\D$, a.k.a atoms), for $i=1,\dots,n$, i.e., 
\begin{equation}
    \x^i \approx \D \z^i \quad \text{subject to} \quad \| \z^i \|_0 \leq k, \quad \forall\ i = 1,\dots,n.\label{eq:dictlearning}
\end{equation}
where the constraint ensures that the sparse code $\z^i$ is at most $k$-sparse. Sparse representations are ubiquitous in science and engineering, with applications in medical imaging~\cite{lustig_sparse_2007, hamalainen_sparse_2013}, image restorations~\cite{mairal2007sparse, dong_nonlocally_2013}, transcriptomics~\cite{cleary2021compressed}, and genomics~\cite{lucas2006sparse} with origins from computational neuroscience~\cite{olshausen1997sparse,olshausen1996emergence}. The formulation in \eqref{eq:dictlearning} leads to a bi-level optimization problem: the \emph{inner} problem performs sparse approximation to estimate the code $\z^i$ given the dictionary $\D$, while the \emph{outer} problem updates the dictionary based on the current code estimates to better represent the data~\cite{tolooshams2023deep}.

Solving this bi-level optimization can be achieved via alternating minimization~\cite{agarwal2016learning}, alternating between the inner and outer problems until a convergence criterion. The inner problem has been extensively studied in the compressed sensing literature~\cite{donoho2006compressed, candes2006robust}. Classical approaches include greedy $\ell_0$-based algorithms~\cite{tropp_greed_2004} such as Matching Pursuit~\cite{MallatMP} and Orthogonal Matching Pursuit~\cite{pati1993omp}, as well as convex relaxation $\ell_1$-based methods~\cite{chen2001atomic, tibshirani1996lasso}. Since sparse recovery lacks a closed-form solution~\cite{natarajan_sparse_1995}, the sparse approximation step typically requires multiple residual-based iterations to converge.

Prior work solves dictionary learning by optimizing the outer problem using closed-form least squares~\cite{agarwal2016learning}, local gradient updates~\cite{chatterji2017alternating}, or sequential methods like MOD~\cite{engan1999mod} and K-SVD~\cite{aharon2006ksvd}. A key bottleneck lies in the repeated solution of the inner sparse coding problem, which typically converges sublinearly~\cite{beck2009fast, moreau_understanding_2017}. To address this, the unrolling literature proposes turning the iterative solver into a neural network, enabling fixed-complexity approximations. This idea, sparked by LISTA~\cite{gregor2010lista}, has been shown to achieve linear convergence~\cite{chen2018unfoldista, ablin2019stepsize} and accelerate both inference and learning~\cite{tolooshams2022stable, malezieux2022understanding}. Unrolling further allows the inner and outer optimization to be directly mapped to forward inference and backward learning in deep networks~\cite{tolooshams2020deep}.

Sparsity has been established as a useful prior for interpretability~\cite{mairal_sparse_2014, lipton_mythos_2017, ribeiro_why_2016}. Building on this principle, recent work has proposed the use of sparse autoencoders (SAEs) to extract human-interpretable features from the internal activations of large language models (LLMs)~\cite{elhage2022toymodelssuperposition, cunningham_sparse_2023, bricken2023monosemanticity, rajamanoharan2024jumping, gao2025scaling}. These methods are motivated by the linear representation hypothesis ~\cite{arora_latent_2016, park_linear_2024,elhage2022toymodelssuperposition}, which posits that internal representations can be modeled as sparse linear combinations of semantic directions. Unlike the classical dictionary learning methods described above, SAEs solve the inner problem in a single iteration. 

\begin{wrapfigure}[33]{r}{0.38\textwidth}
\hspace{-12pt}
    \begin{minipage}{0.40\textwidth}
    \vspace{-35pt}
        \begin{algorithm}[H]
        \caption{Matching Pursuit Sparse Autoencoders (MP-SAE)}
        \label{alg:mpsae}
        \begin{algorithmic}
        \STATE \textbf{Input:} Dictionary $\D$, bias $\bm b_{\text{pre}}$, data $\bm x$, steps $T$
        \STATE Initialize residual: $\bm r^{(0)} = \bm x - \bm b_{\text{pre}}$
        \STATE Initialize reconstruction: $\hat{\bm x}^{(0)} = \bm b_{\text{pre}}$
        \STATE Initialize sparse code: $\bm z^{(0)} = \bm 0$
        \FOR{$t = 1, \ldots, T$}
            \STATE $j^{(t)} = \argmax_{j=\{1,\ldots,p\}} (\bm \D^\top\bm r^{(t-1)})_j $
            \STATE $\z^{(t)}_{j^{(t)}} = \D_{j^{(t)}}^\top \res^{(t-1)}$
            \STATE $\hat{\bm x}^{(t)} = \hat{\bm x}^{(t-1)} + \z^{(t)}_{j^{(t)}} \D_{j^{(t)}}$
            \STATE $\res^{(t)} = \res^{(t-1)} - \z^{(t)}_{j^{(t)}} \D_{j^{(t)}}$
        \ENDFOR
        \STATE \textbf{Output:} Sparse code $\bm z=\sum_{t=1}^T \z^{(t)}$ \\
        \vspace{1mm}
        \hspace{1.3cm}Reconstruction $\hat{\bm x} = \hat{\bm x}^{(T)}$
        \end{algorithmic}
        \end{algorithm}
        \vspace{-3mm}
        \centering
        \includegraphics[width=0.97\linewidth]{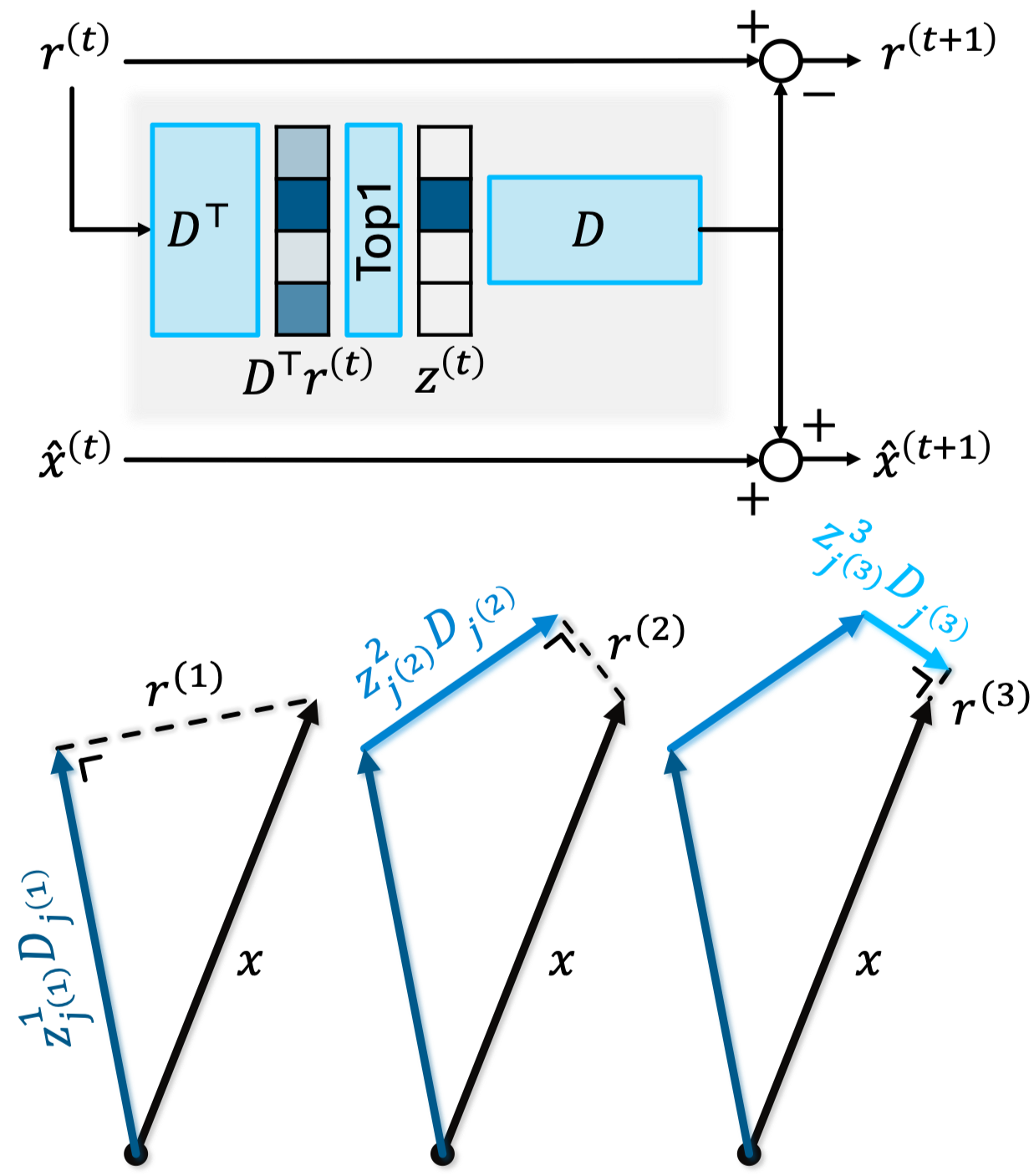}
        \captionof{figure}{\textbf{MP-SAE.} 
        Top : Full forward pass. Bottom: One encoder iteration.}
        \label{fig:mp}
    \end{minipage}
\end{wrapfigure}

Although SAEs are widely used for model interpretability~\cite{kim2018interpretability, cunningham_sparse_2023}, they have rarely been evaluated in small-scale, controlled settings,  despite their strong ties to classical dictionary learning.  Among the few studies, \cite{hindupur2025projecting} demonstrate that SAEs impose structural assumptions that shape what they can and cannot detect.

\textbf{Our Contributions}\quad In this work, we revisit the relationship between SAEs and dictionary learning by testing common interpretability SAEs architectures on MNIST, a controlled and well-established dataset in sparse coding literature. Despite similarities in architecture, we find that SAEs with different sparsity mechanisms yield structurally distinct dictionaries. These differences may have important implications for interpretability.

We demonstrate that shallow SAEs implicitly favor near-orthogonal dictionaries, due to their one-shot sparse inference. To investigate this further, we introduce MP-SAE, which unrolls Matching Pursuit into a sequential, residual-guided inference process that operates effectively in regimes with highly correlated features and is supported by convergence guarantees. MP-SAE learns a globally correlated set of atoms but uses a low-redundancy subset to represent each input, a property missing in current shallow SAEs.

Compared to shallow SAEs, MP-SAE yields more expressive representations and naturally constructs a representational hierarchy—first selecting atoms that capture coarse structure, then adding finer details. This coarse-to-fine behavior may lead to more interpretable representations, as it mirrors the hierarchical organization of real-world features~\cite{bussmann2025learning}.

\section{Background}

\textbf{Representation Hypotheses}\quad Efforts to understand neural representations through interpretability are often motivated by two guiding hypotheses: the \textit{Linear Representation Hypothesis} and the \textit{Superposition Hypothesis}~\cite{arora_latent_2016, olah_zoom_2020, park_linear_2024, elhage2022toymodelssuperposition}. These suggest that internal activations of large deep neural networks can be expressed as linear combinations of human-interpretable concept directions $\D$. In practice, the number of possible concepts \(p\) far exceeds the dimensionality \(m\) of the representation: $m \ll p$, leading to superposition—multiple concepts overlapping within the same activation~\cite{elhage2022toymodelssuperposition}. Despite this, meaningful disentanglement is possible under the sparsity assumption: $\|\bm{z}\|_0 \leq k$, where only a small number \(k \ll p\) of concepts are active~\cite{donoho2006compressed}.

\textbf{Sparse Concept Extraction as Dictionary Learning}\quad As formalized in~\cite{fel_holistic_2023}, the task of concept extraction in interpretability can be cast as a dictionary learning problem: learn a set of interpretable directions \(\D\) such that activations \(\bm{x}\) can be approximated by sparse linear combinations with sparse code \(\bm{z}\) (See equation \ref{eq:dictlearning}). In practice, this is most often implemented using shallow SAEs, which have been shown to extract meaningful and monosemantic concepts across a variety of architectures and domains~\cite{elhage2022toymodelssuperposition, cunningham_sparse_2023, bricken2023monosemanticity}.

\textbf{Sparse Autoencoders}\quad Given an input \(\bm{x} \in \mathbb{R}^m\), an SAE computes a sparse code using an encoder $\bm{z} = \sigma(\W^\top (\bm{x} - \bm{b}_{\text{pre}}) + \bm{b})$ and reconstructs the data as $\hat{\bm{x}} = \D \bm{z} + \bm{b}_{\text{pre}}$, where $\W \in \mathbb{R}^{m \times p}$ and $\bm{b} \in \mathbb{R}^p$ are the encoder parameters, $\bm{b}_{\text{pre}} \in \mathbb{R}^m$ is a bias, and $\D \in \mathbb{R}^{m \times p}$ is the learned dictionary of interest with normalized atoms (i.e., $\|\D_{i}\|_2=1$). The nonlinearity \(\sigma(\cdot)\) enforces sparsity; common choices include ReLU, TopK, and JumpReLU. These models are trained to minimize a sparsity-augmented reconstruction loss: $\mathcal{L} = \|\bm{x} - \hat{\bm{x}}\|_2^2 + \lambda\, \mathcal{R}(\bm{z}) + \alpha\, \mathcal{L}_{\text{aux}}$, where $\mathcal{R}(\bm{z})$ promotes sparsity, e.g., via $\ell_1$ regularization for ReLU~\cite{cunningham_sparse_2023, bricken2023monosemanticity} or target-$\ell_0$ penalties~\cite{rajamanoharan2024jumping}.

\textbf{Orthogonality and Limitations of Shallow Recovery}\quad Sparse recovery (i.e., the inner problem optimization) in one iteration is theoretically guaranteed only when the dictionary \(\D\) is sufficiently incoherent, i.e., when the mutual coherence \(\mu(\D) = \max_{i \neq j} |\D_i^\top \D_j|\) is small~\cite{makhzani2014ksae, arora2015sparsecoding}. However, concepts underlying natural data may exhibit high coherence. This limitation motivates the use of unrolled sparse coding networks, which address the inner problem through multiple iterative updates, progressively refining the estimate of the sparse code $z$ as in classical sparse recovery methods. Shallow SAEs can be seen as approximating unrolled networks by performing only a single iteration: they rely on the same nonlinearities -- JumpReLU implementing hard thresholding as in NOODL~\cite{rambhatla2018noodl}, and ReLU implementing soft thresholding as in PUDLE~\cite{tolooshams2022stable} -- but do not perform iterative updates.

\section{Unrolling Matching Pursuit into MP-SAE}

We propose the \textit{Matching Pursuit Sparse Autoencoder} (MP-SAE), which unrolls Matching Pursuit (MP)~\cite{MallatMP} into an iterative sparse autoencoder architecture. Previously used for image restoration tasks under the name Learned-MP~\cite{khatib2021learned}, our work focuses on the dictionary properties learned by MP-SAE and compares them to those obtained with shallow SAEs. 

\textbf{Iterative Inference via Residual Updates} \quad MP-SAE can be seen as an unrolled variant of TopK. Instead of selecting $k$ atoms in a single iteration, as in TopK, it selects one atom at each of the $k$ iterations. This allows the model to subtract the contribution of each selected atom from the input before the next selection, so that subsequent atoms are chosen to explain what remains unexplained. This unexplained part is captured by the residual, defined as the difference between the input $\bm{x}$ and its current reconstruction $\hat{\bm{x}}$, i.e., $\bm{r} = \bm{x} - \hat{\bm{x}}$. Initially, the residual is equal to the input (or $\bm{x} - \bm{b}_{\text{pre}}$).

\begin{wrapfigure}[24]{r}{0.39\textwidth}
    \vspace{-2mm}
    \hspace{-5mm}
    \centering
    \includegraphics[width=1.06\linewidth]{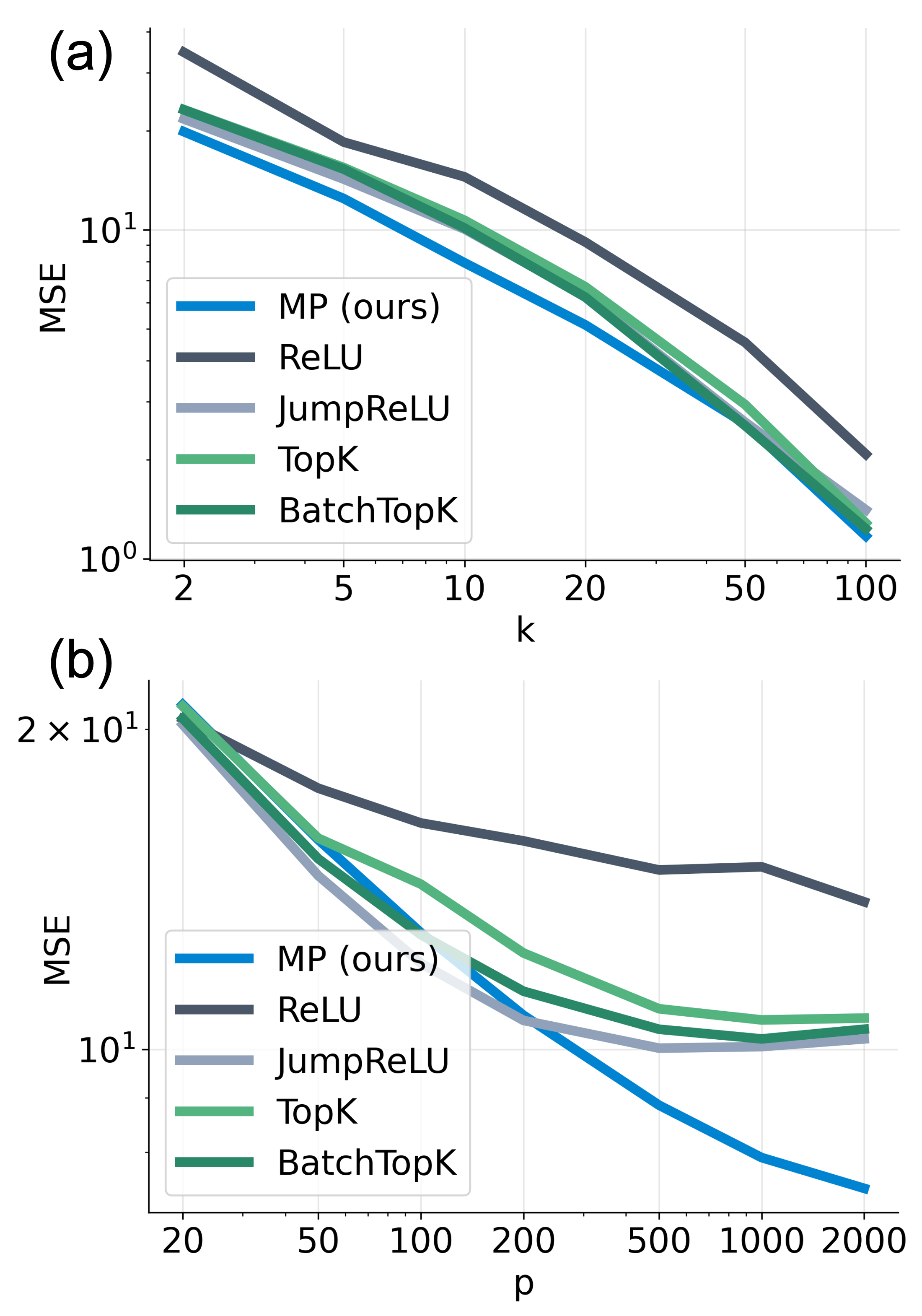}
    \vspace{-6mm}
    \caption{\textbf{Expressivity.}}
    \label{mnist_rec}
\end{wrapfigure}

At each iteration, MP greedily selects the dictionary atom that best aligns with the current residual. This is done by computing the inner product between the residual and each atom, and selecting the one with the highest projection. Once the best-matching atom is selected, the algorithm projects the residual onto that atom to determine its contribution to the reconstruction. This contribution is then added to the current approximation of the input and subtracted from the residual. Over time, this greedy procedure iteratively reduces the residual and improves the reconstruction, step by step (see Figure~\ref{fig:mp} and Algorithm~\ref{alg:mpsae}).

\textbf{Theoretical Properties of Matching Pursuit} \quad At each step, the residual \(\bm{r}^{(t)}\) is orthogonal to the selected atom \(\D_{j^{(t)}}\) (Proposition~\ref{prop:ortho}). Moreover,  the norm of the residual decreases monotonically at each iteration (Proposition~\ref{prop:monotonic}) and converges asymptotically to the component of the input orthogonal to the span of the dictionary \(\D\) (Proposition~\ref{prop:convergence}). 

Because each selected atom is subtracted from the residual at each iteration, the algorithm naturally explores new directions -- each selection is driven by what remains unexplained, pushing the model toward atoms that are less redundant and more complementary. This residual-driven mechanism improves diversity in the selected features and enhances robustness in the presence of dictionary coherence. Indeed, MP can achieve exponential convergence and accurately recover active components even when the dictionary is highly coherent, as long as it satisfies a block-incoherence condition, where correlations are localized within small groups of atoms~\cite{peotta2007matching}.

\section{Results}

We evaluate MP-SAE against four shallow SAEs: ReLU~\cite{elhage2022toymodelssuperposition, cunningham_sparse_2023}, JumpReLU~\cite{rajamanoharan2024jumping}, TopK~\cite{gao2025scaling}, and BatchTopK~\cite{bussmann2024batchtopk}, using the MNIST dataset. Additional results on large vision models, including expressivity and coherence, are provided in Appendix~\ref{sec:lvm}.

\textbf{Expressivity}\quad We assess reconstruction performance by varying two key parameters: the sparsity level \(k\) and the dictionary size \(p\). Figure~\ref{mnist_rec}(a) shows that, when fixing \(p = 1000\) and varying \(k\), MP-SAE is consistently more expressive across sparsity levels—despite having half the capacity of other SAEs due to the absence of encoder parameters from weight tying. As shown in Figure~\ref{mnist_rec}(b), when \(k = 10\) is fixed and \(p\) is swept, MP-SAE continues to improve as the dictionary size grows, while shallow SAEs plateau. This highlights the efficiency of MP-SAE in leveraging additional capacity under the same sparsity constraint.

\textbf{Learned Concepts}\quad In the following, we focus on dictionaries trained with \(k = 10\) and \(p = 1000\). The top row in Figure~\ref{fig:atoms} shows 25 atoms with the highest activation frequency for each SAE. All methods except ReLU appear to learn pen-stroke-like patterns, while ReLU learns atoms resembling full digits. The pen strokes learned by MP-SAE appear more precise. Interestingly, all shallow methods include a ``negative'' atom that closely resembles \(-\bm{b}_{\text{pre}}\), which is the most frequently activated atom in ReLU and JumpReLU.

When sorting atoms by their average activation value $\E(\z_j)$ rather than frequency $\sum_i \z^i_j \neq 0$, a structural shift unique to MP-SAE emerges. As shown in the bottom row of Figure~\ref{fig:atoms}, the most heavily weighted atoms in MP-SAE resemble clean, idealized digit prototypes, in contrast to the more detailed pen strokes observed among its most frequently activated atoms. In comparison, shallow SAEs exhibit little variation between these two rankings; ReLU continues to activate full-digit atoms, while the others still primarily activate stroke-like patterns.
\vspace{-0.3cm}

\begin{figure}[h]
    \centering
    \includegraphics[width=\linewidth]{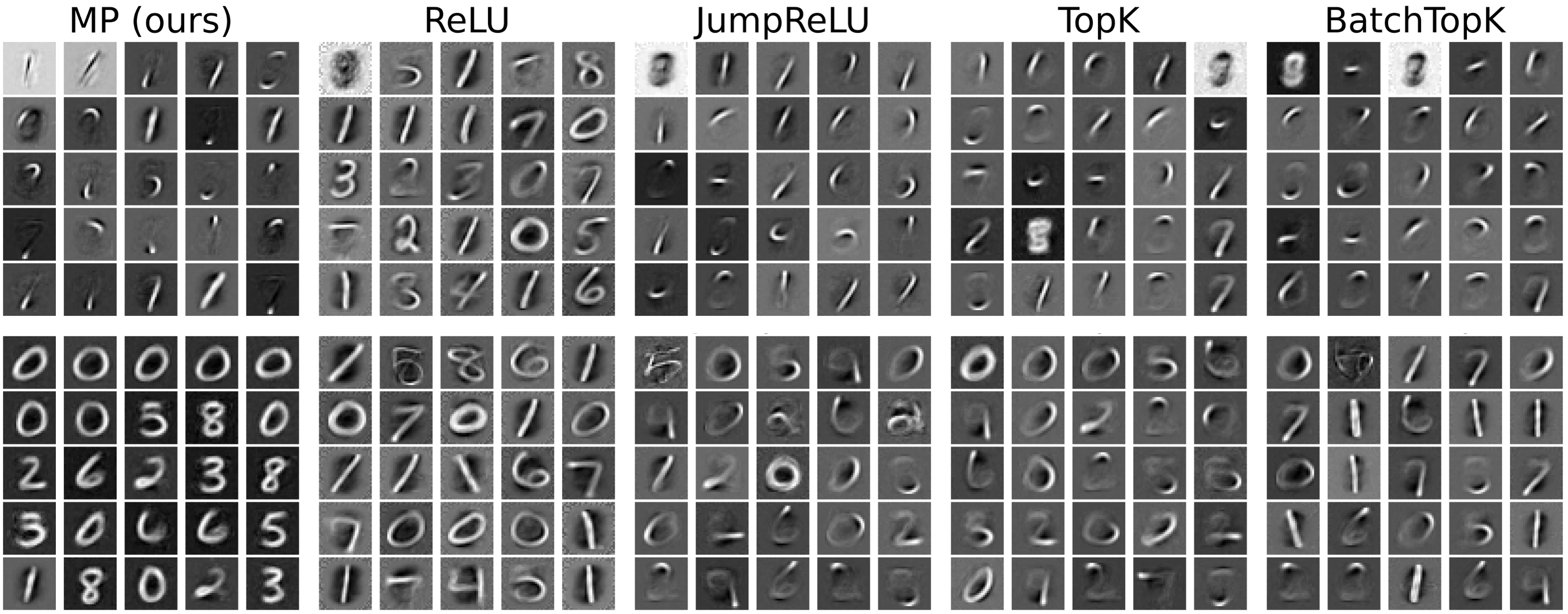} 
    \caption{\textbf{Feature Selection vs.\ Activation Levels}. Top: atoms with highest activation frequency ($\ell_0$). Bottom: atoms with highest activation $\E[\z_j]$ ($\ell_1$).}
    \label{fig:atoms}
    \vspace{-0.3cm}
\end{figure}

\textbf{Hints of Hierarchy}\quad To capture this shift, the distributions of activation frequency and average activation value $\mathbb{E}[\z_j]$ are shown in Figure~\ref{fig:activations}(a). JumpReLU, TopK, and BatchTopK exhibit high variance in activation frequency—some atoms are rarely used, while others are activated very frequently. In contrast, the variance in activation values remains low, with slightly higher values observed for the more frequently used atoms. By comparison, ReLU activates its dictionary atoms uniformly, both in frequency and magnitude. MP displays a perpendicular trend to the other \(\ell_0\)-based methods and JumpReLU: its atoms are activated with roughly equal frequency—similar to ReLU—but their activation values vary widely. Some atoms contribute much more to the reconstruction than others. A subtle inverse relationship is also observed: atoms with higher activation values tend to be used less frequently.

This hierarchical behavior is supported by Figure~\ref{fig:activations}(b), which shows that atoms with higher average activation values tend to be selected in the early layers of MP-SAE (the first iterations of Matching Pursuit). These atoms correspond to digit-like patterns in the bottom row of Figure~\ref{fig:atoms}, and are refined by later atoms capturing more localized pen strokes (top row). This progression suggests a hierarchical structure in MP-SAE, reconstructing from coarse to fine features (see Appendix~\ref{progressive_rec}).

To enable comparison with shallow encoders, atoms are reordered by activation value to simulate a sequential selection. \(\ell_0\)-based methods show a concentration of atoms toward the end, likely due to the auxiliary loss. This also highlights that ReLU exhibits an amplitude bias. 

\begin{figure}[H]
    \centering
    \includegraphics[width=0.75\linewidth]{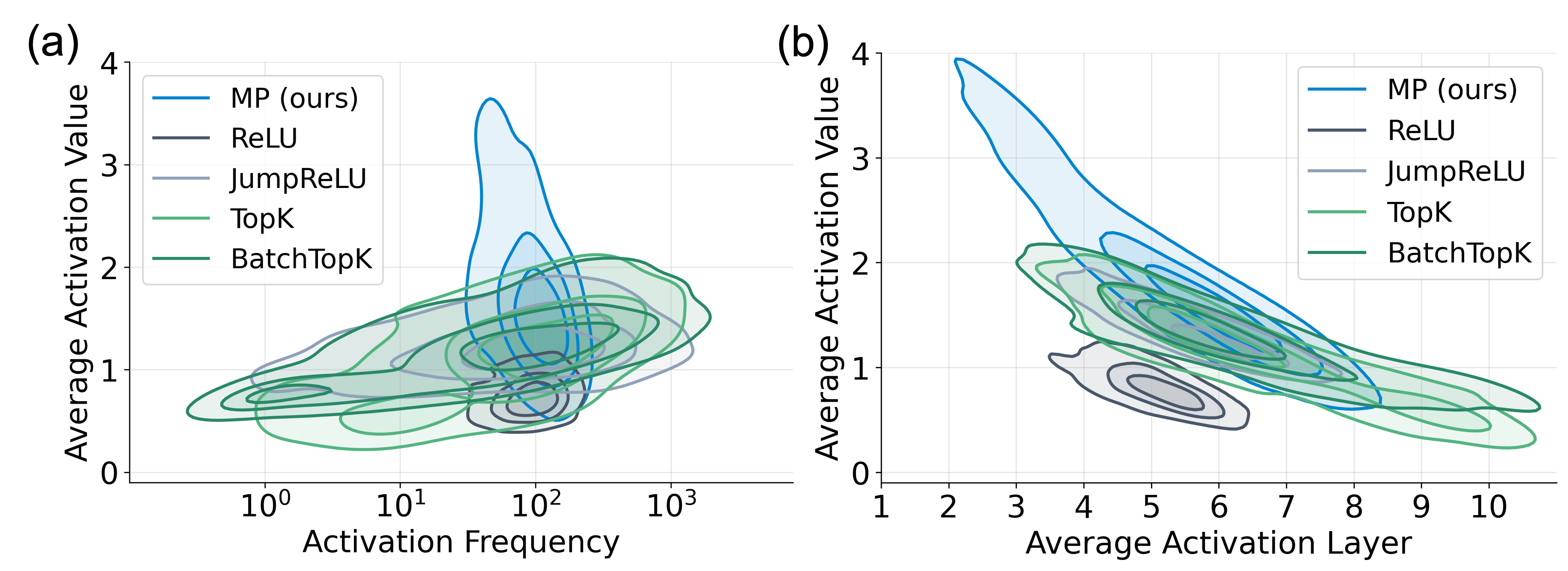}
    \vspace{-2mm}
    \caption{\textbf{Activation distributions}.}
    \label{fig:activations}
    \vspace{-3mm}
\end{figure}

\textbf{Coherence}\quad Finally, we assess the coherence of both the learned dictionary and the atoms selected at inference time using the Babel score~\cite{tropp_greed_2004}. Unlike mutual coherence, which captures only the maximum pairwise similarity between atoms, the Babel function provides a more comprehensive measure of redundancy by quantifying cumulative coherence. Given a dictionary $\D = [\D_1, \dots, \D_p] \in \mathbb{R}^{m \times p}$ with unit-norm columns, the Babel function of order $r$ is defined as:
\begin{equation*}
    \mu_1(r) = \max_{\substack{S \subset [p], |S| = r}} \left( \max_{j \notin S} \sum_{i \in S} \left| \D_i^\top \D_j \right| \right)
\end{equation*}
The value $\mu_1(r)$ reflects how well a single atom can be approximated by a group of $r$ others—lower values indicate lower redundancy. Further details are provided in Appendix~\ref{sec:lvm}.

\begin{figure}[H]
    \centering
    \includegraphics[width=0.75\linewidth]{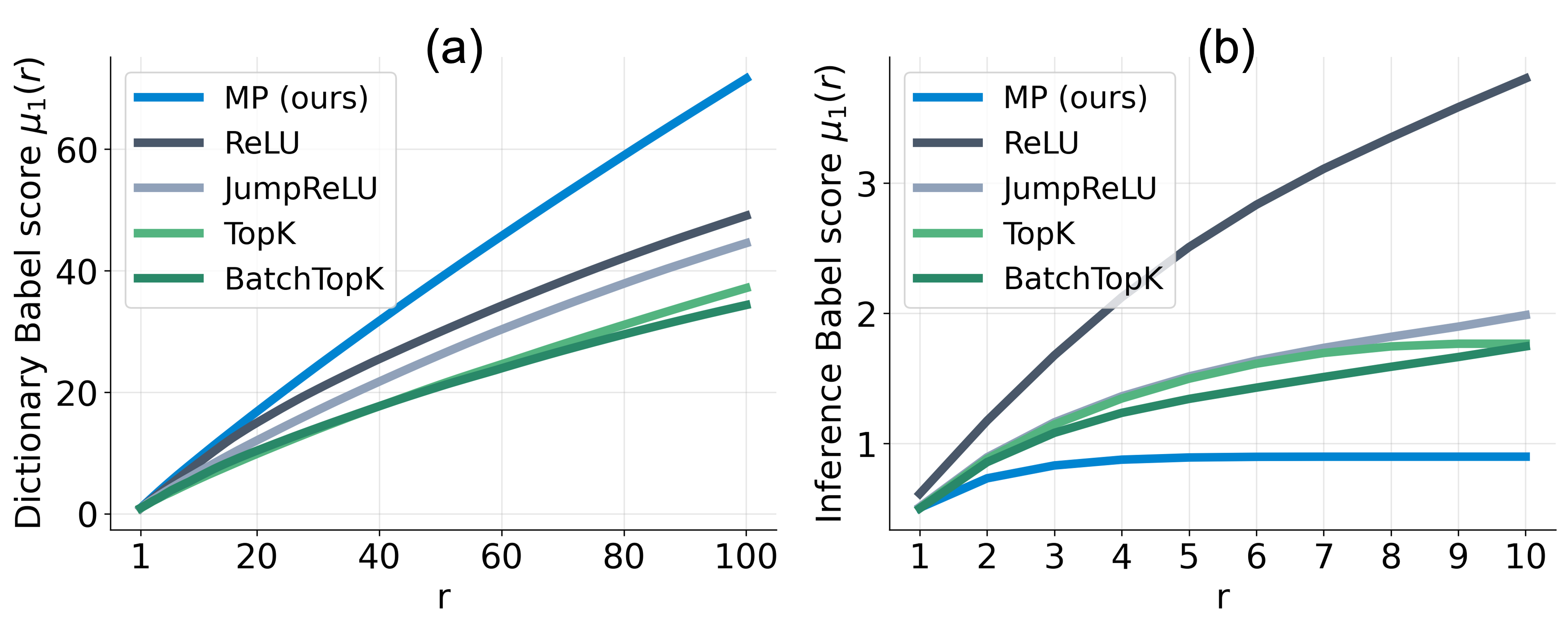}
    \vspace{-2mm}
    \caption{\textbf{Coherence analysis of learned concepts.}}
    \label{fig:babel}
    \vspace{-3mm}
\end{figure}
Figure~\ref{fig:babel}(a) shows that MP-SAE exhibits a more coherent dictionary than the shallow SAEs. However, as shown in Figure~\ref{fig:babel}(b), it selects more incoherent atoms at inference. This highlights MP's ability to draw incoherent subsets from a globally coherent dictionary. Interestingly, for shallow SAEs, the trends for the learned dictionary and the selected atoms align: ReLU consistently exhibits the highest coherence, while TopK remains the least coherent. This suggests that shallow SAEs are constrained to select more correlated atoms when the dictionary itself is more coherent.

\section{Conclusion}

We introduce MP-SAE and show through small-scale experiments on MNIST that it improves expressivity, learns hierarchical features, and overcomes coherence limitations of shallow SAEs. These experiments also reveal distinct representation behaviors across SAEs and offer a foundation for building more interpretable sparse autoencoders.

\newpage

\bibliography{main}
\bibliographystyle{icml2025}

\newpage

\section{Sequential Reconstruction on MNIST}\label{progressive_rec}

Figure~\ref{fig:prog} illustrates how each model reconstructs the input step by step, starting from the pre-activation bias \(\smlb_{\text{pre}}\). For shallow SAEs, atoms are reordered by their activation values in the sparse code \(\z\) to simulate a sequential inference process. Note that for ReLU, JumpReLU, and BatchTopK, the number of selected atoms may differ from \(k = 10\), as these methods do not enforce a fixed sparsity level.

MP-SAE exhibits a clear coarse-to-fine reconstruction pattern: with just two atoms, the model already recovers the input's global structure—a zero with an internal pen stroke. Subsequent atoms progressively refine the digit’s contour using precise pen-stroke components, highlighting the hierarchical behavior of MP-SAE.

In contrast, ReLU fails to recover the inner stroke, likely because its dictionary contains few atoms resembling pen strokes and is dominated by full-digit prototypes. JumpReLU, TopK, and BatchTopK reconstruct the digit by combining multiple pen-stroke atoms, both for the outer zero shape and the internal stroke, relying on distributed, part-based representations.

\begin{figure}[H]
    \centering
    \includegraphics[width=\linewidth]{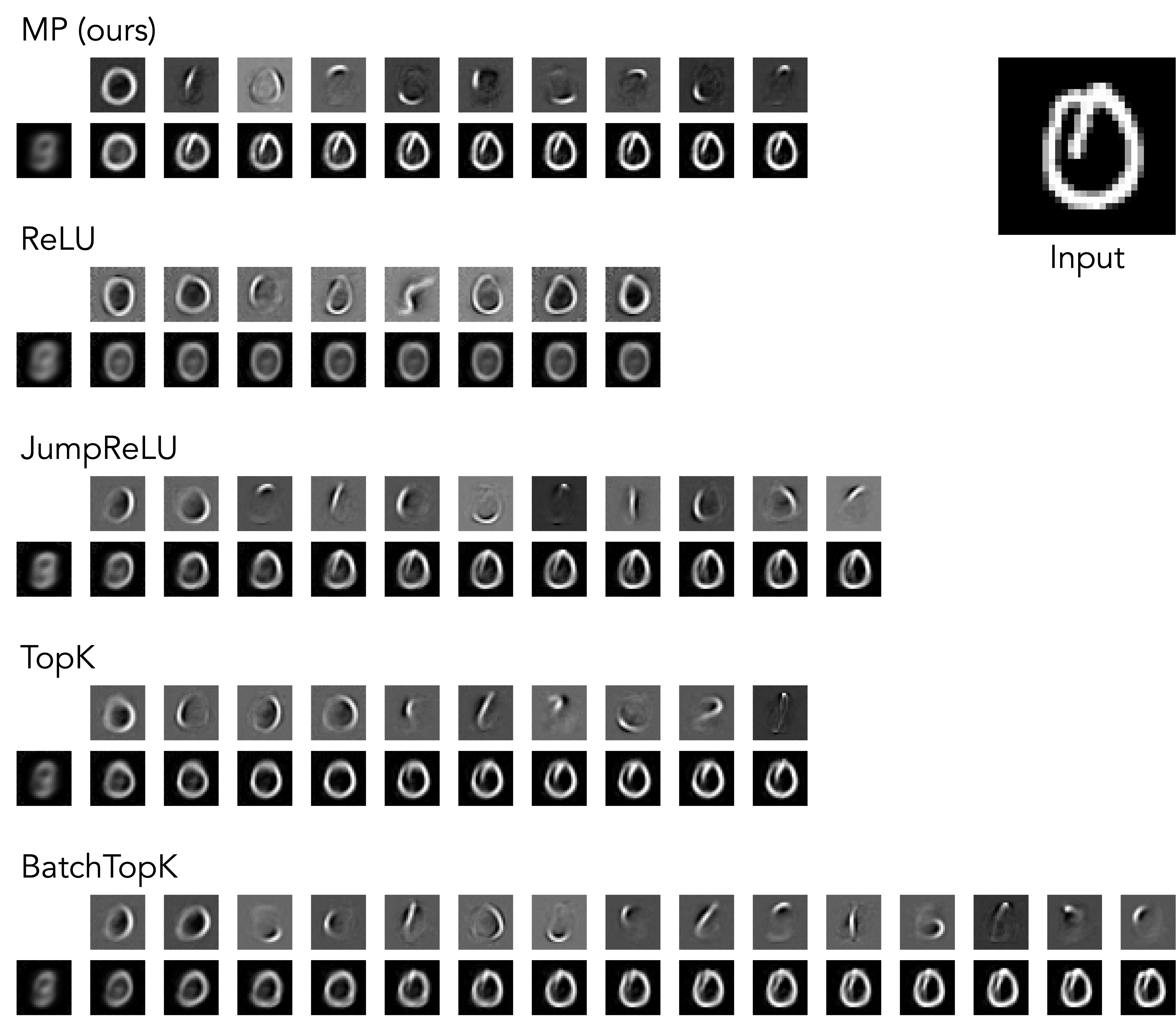}
    \caption{Example of Sequential Reconstruction for \(k = 10\).}
    \label{fig:prog}
\end{figure}

\section{Results on Large Vision Models}\label{sec:lvm}

We evaluate MP-SAE on large vision models and compare it to three shallow SAEs: Vanilla (ReLU), TopK, and BatchTopK. Our results show that the findings observed on MNIST generalize to this setting.

\textbf{Expressivity} \quad We first assess the representational expressivity of MP-SAE relative to standard SAEs. Figure~\ref{fig:expressivity} presents the Pareto frontier obtained by varying the sparsity level while keeping the dictionary size $p$ fixed. Across all evaluated models—SigLIP~\cite{zhai2023sigmoid}, DINOv2~\cite{oquab2023dinov2}, CLIP~\cite{radford2021learning}, and ViT~\cite{dosovitskiy2020image}—MP-SAE consistently achieves higher $R^2$ values at similar sparsity levels, indicating more efficient reconstructions.

Training was conducted for 50 epochs using the Adam optimizer, with an initial learning rate of $5\cdot10^{-4}$ decayed to $10^{-6}$ via cosine annealing with warmup. All SAEs used an expansion factor of 25 ($p = 25m$). Models were trained on the ImageNet-1k~\cite{imagenet2009} training set, using frozen features from the final layer of each backbone. For ViT-style models (e.g., DINOv2), we included both the CLS token and all spatial tokens (approximately 261 tokens per image for DINOv2), resulting in roughly 25 billion training tokens overall.
\vspace{-2mm}
\begin{figure}[H]
\centering
\includegraphics[width=\linewidth]{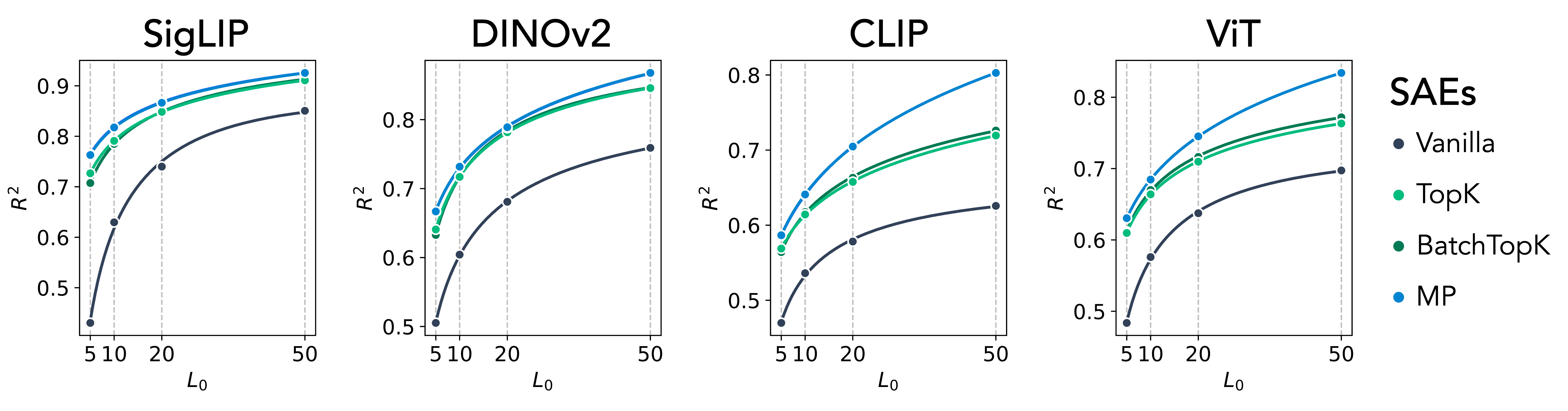}
\vspace{-8mm}
\caption{\textbf{MP-SAE recovers more expressive atoms than standard SAEs.}
Reconstruction performance ($R^2$) as a function of sparsity level across four pretrained vision models: SigLIP, DINOv2, CLIP, and ViT. MP-SAE consistently yields higher $R^2$ at comparable sparsity, suggesting more informative and efficient decompositions.
}
\label{fig:expressivity}
\end{figure}
\vspace{-4mm}
\textbf{Coherence}\quad To evaluate coherence beyond pairwise similarity, we use the Babel function~\cite{tropp_greed_2004}, a standard metric in sparse approximation that captures cumulative interference among dictionary atoms. Mutual coherence reflects only the maximum absolute inner product between pairs of atoms; as a result, it can be misleadingly high if just one pair has a large inner product while the rest are small. In contrast, the Babel function measures the total interference between an atom and a group of others, offering a more comprehensive view of redundancy. Intuitively, $\mu_1(r)$ quantifies how well a single atom can be approximated by a group of $r$ others. Lower values indicate lower redundancy, reflecting the degree of overlap between an atom and its $r$ nearest neighbors in the dictionary.

Figure~\ref{fig:babel_lvm} reports $\mu_1(r)$ for both the full dictionary (top) and for subsets of atoms co-activated at inference (bottom). As observed on MNIST, MP-SAE learns globally coherent dictionaries with low Babel scores. However, the subsets of atoms selected during inference exhibit higher Babel values, indicating local incoherence. This duality in coherence persists even when training on large vision models.
\vspace{-4mm}
\begin{figure}[H]
\centering
\includegraphics[width=0.95\linewidth]{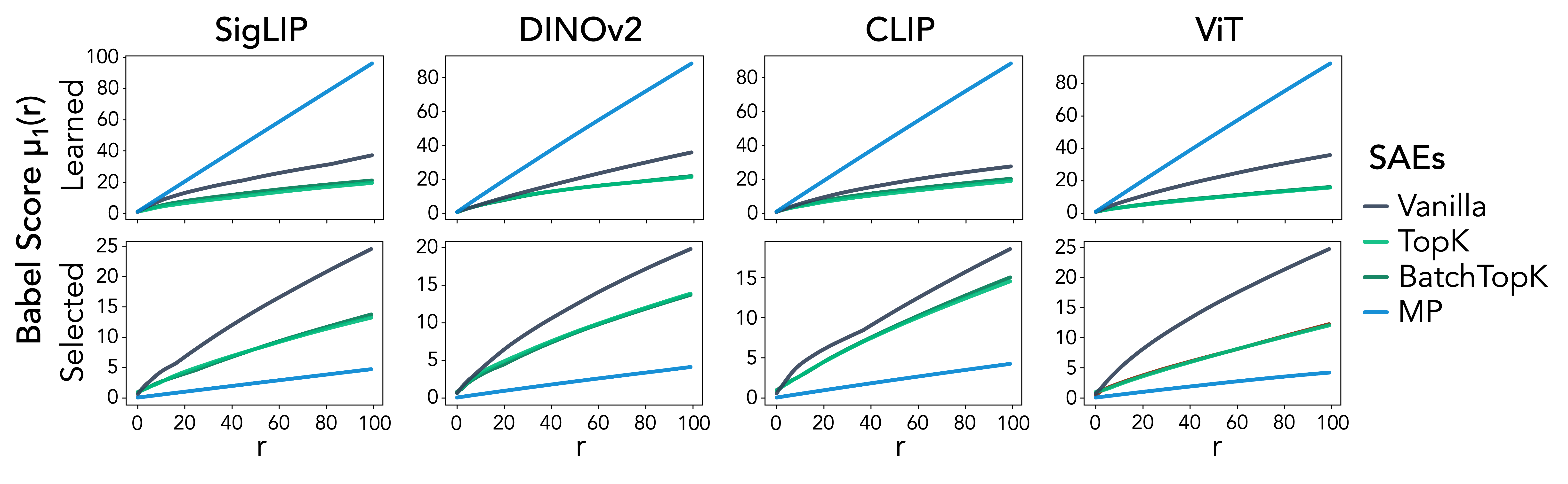}
\vspace{-4mm}
\caption{\textbf{MP-SAE learns more coherent dictionaries but selects incoherent atoms.}
Babel scores for the full dictionaries (top) and co-activated subsets at inference time (bottom).
}
\label{fig:babel_lvm}
\end{figure}

\section{Theoretical Properties of Matching Pursuit}\label{app:ortho_proof}

We restate three foundational properties of Matching Pursuit—originally established in the sparse coding literature~\cite{MallatMP}—and interpret them in the context of sparse autoencoders. These properties help elucidate the structure and dynamics of the representations learned by MP-SAE.

\begin{itemize}
    \item \textbf{Stepwise orthogonality} (Proposition~\ref{prop:ortho}): at each iteration, the residual becomes orthogonal to the atom most recently selected by the greedy inference rule. This sequential orthogonalization mechanism gives rise to a locally disentangled structure in the representation and reflects the conditional independence induced by MP-SAE inference.
    \item \textbf{Monotonic decrease of residual energy} (Proposition~\ref{prop:monotonic}): the $\ell_2$ norm of the residual decreases whenever it retains a nonzero projection onto the span of the dictionary. This guarantees that inference steps lead to progressively refined reconstructions, and enables sparsity to be adaptively tuned at inference time without retraining.
    \item \textbf{Asymptotic convergence} (Proposition~\ref{prop:convergence}): in the limit of infinite inference steps, the reconstruction converges to the orthogonal projection of the input onto the subspace defined by the dictionary. Thus, MP-SAE asymptotically recovers all structure that is representable within its learned basis.\\
\end{itemize}

\begin{proposition}[Stepwise Orthogonality of MP Residuals]\label{prop:ortho}
Let $\res^{(t)}$ denote the residual at iteration $t$ of MP-SAE inference, and let $j^{(t)}$ be the index of the atom selected at step $t$. If the column $j^{(t)}$ of the dictionary $\D$ satisfy $\|\D_{j^{(t)}}\|_2 = 1$, then the residual becomes orthogonal to the previously selected atom:
\[
\D_{j^{(t)}}^\top \res^{(t)} = 0.
\]
\end{proposition}

\begin{proof}
This follows from the residual update:
\[
\res^{(t)} = \res^{(t-1)} - \D_{j^{(t)}} \z_{j^{(t)}}^{(t)},
\]
with $\z_{j^{(t)}}^{(t)} = \D_{j^{(t)}}^\top \res^{(t-1)}$. Taking the inner product with $\D_{j^{(t)}}$ gives:
\[
\D_{j^{(t)}}^\top \res^{(t)} = \D_{j^{(t)}}^\top \res^{(t-1)} - \|\D_{j^{(t)}}\|^2 \z_{j^{(t)}}^{(t)}  = \z_{j^{(t)}}^{(t)}  - \z_{j^{(t)}}^{(t)} = 0. \qedhere
\]
\end{proof}

This result captures the essential inductive step of Matching Pursuit: each update removes variance along the most recently selected atom, producing a residual that is orthogonal to it. Applied iteratively, this localized orthogonality promotes the emergence of conditionally disentangled structure in MP-SAE. In contrast, other sparse autoencoders lack this stepwise orthogonality mechanism, which helps explain the trend observed in the Babel function during inference in Figure~\ref{fig:babel}.

\begin{proposition}[Monotonic Decrease of MP Residuals]\label{prop:monotonic}
Let $\res^{(t)}$ denote the residual at iteration $t$ of MP-SAE inference, and let $\z_{j^{(t)}}^{(t)}$ be the nonzero coefficient selected at that step, Then the squared residual norm decreases monotonically:
\[
\|\res^{(t)}\|_2^2 - \|\res^{(t-1)}\|_2^2 = -\|\D_{j^{(t)}} \z_{j^{(t)}}^{(t)}\|_2^2 \leq 0.
\]
\end{proposition}

\begin{proof}
From the residual update:
\[
\res^{(t)} = \res^{(t-1)} - \D_{j^{(t)}} \z_{j^{(t)}}^{(t)},
\]
we can rearrange to write:
\[
\res^{(t-1)} = \res^{(t)} + \D_{j^{(t)}} \z_{j^{(t)}}^{(t)}.
\]
Taking the squared norm of both sides:
\begin{align*}
\|\res^{(t-1)}\|_2^2 &= \|\res^{(t)} + \D_{j^{(t)}} \z_{j^{(t)}}^{(t)}\|_2^2 \\
&= \|\res^{(t)}\|_2^2 + 2 \langle \res^{(t+1)}, \D_{j^{(t)}} \rangle \z_{j^{(t)}}^{(t)} + \|\D_{j^{(t)}} \z_{j^{(t)}}^{(t)}\|_2^2.
\end{align*}
By Proposition~\ref{prop:ortho}, the cross term vanishes:
\[
\langle \res^{(t)}, \D_{j^{(t)}} \rangle = 0,
\]
yielding:
\[
\|\res^{(t-1)}\|_2^2 = \|\res^{(t)}\|_2^2 + \|\D_{j^{(t)}} \z_{j^{(t)}}^{(t)}\|_2^2. \qedhere
\]
\end{proof}

The monotonic decay of residual energy ensures that each inference step yields an improvement in reconstruction, as long as the residual lies within the span of the dictionary. Crucially, this property enables MP-SAE to support adaptive inference-time sparsity: the number of inference steps can be varied at test time—independently of the training setup—while still allowing the model to progressively refine its approximation. 

\begin{wrapfigure}[15]{r}{0.45\linewidth}
    \centering
    \includegraphics[width=1\linewidth]{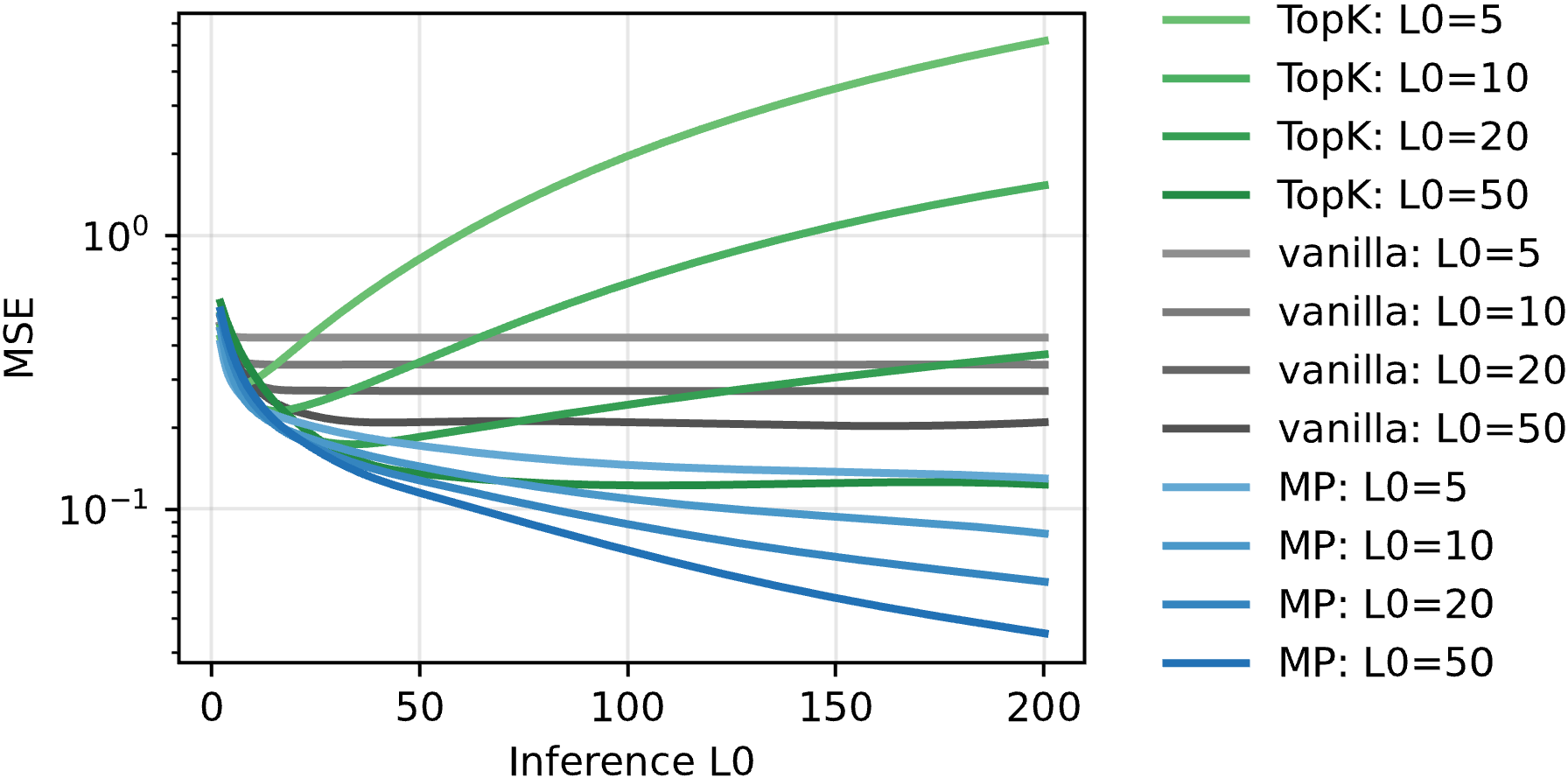}
    \caption{\textbf{Reconstruction error vs.\ inference-time sparsity \(k\).}}
    \label{fig:inference_at_k}
    \vspace{-10pt}
\end{wrapfigure}

As shown in Figure~\ref{fig:inference_at_k}, MP-SAE exhibits a continuous decay in reconstruction error—a behavior explained by the proposition and not guaranteed by other sparse autoencoders. All models were trained on DINOv2 representations with different training-time \(\ell_0\) sparsity levels. At inference, the sparsity \(k\) is varied to assess generalization. MP-SAE shows monotonic improvement, as guaranteed by Proposition~\ref{prop:monotonic}. This stands in contrast to TopK-based SAEs, which often degrade under sparsity mismatch: when trained with fixed \(k\), the decoder implicitly specializes to superpositions of exactly \(k\) features, leading to instability—particularly when the inference-time \(k\) is much larger than the training value. ReLU-based SAEs, by contrast, cannot expand their support beyond the features activated during training and thus exhibit flat or plateaued performance as \(k\) increases.

\begin{proposition}[Asymptotic Convergence of MP Residuals]\label{prop:convergence}
Let $\hat{\x}^{(t)} = \x - \res^{(t)}$ denote the reconstruction at iteration $t$, and let $\mathbf{P}_{\D}$ be the orthogonal projector onto $\mathrm{span}(\D)$. Then:
\[
\lim_{t \to \infty} \| \hat{\x}^{(t)} - \mathbf{P}_{\D} \x \|_2 = \lim_{t \to \infty} \| \mathbf{P}_{\D} \res^{(t)} \|_2 = 0.
\]
\end{proposition}

This convergence result is formally established in the original Matching Pursuit paper ~\cite{MallatMP}[Theorem 1]. This result implies that MP-SAE progressively reconstructs the component of $\x$ that lies within the span of the dictionary, converging to its orthogonal projection in the limit of infinite inference steps. When the dictionary is complete (i.e., $\mathrm{rank}(\D) = m$), this guarantees convergence to the input signal $\x$.

\end{document}